\documentclass{article} % for initial submission
\usepackage{arxiv}
\usepackage[american]{babel}
\usepackage{natbib} % has a nice set of citation styles and commands
\usepackage{mathtools} % amsmath with fixes and additions
\usepackage{booktabs} % commands to create good-looking tables
\usepackage{tikz} % nice language for creating drawings and diagrams

\usepackage{graphicx}
\usepackage{subfigure}
\usepackage{booktabs} % for professional tables
\usepackage[ruled,vlined]{algorithm2e}

\usepackage{hyperref}

\usepackage{hyperref}       % hyperlinks
\usepackage{url}            % simple URL typesetting
\usepackage{amsfonts}       % blackboard math symbols
\usepackage{nicefrac}       % compact symbols for 1/2, etc.
\usepackage{microtype}      % microtypography
\usepackage{lipsum}
\usepackage{xcolor}

\usepackage{multicol}
\usepackage{multirow}
\usepackage{amsmath}

\newcommand{\Rset}{\mathbb R}

\newcommand{\indicator}[1]{ {\mathbb I}[#1] }

\DeclareMathOperator*{\argmax}{arg\,max}

\newcommand{\BlackBox}{\rule{1.5ex}{1.5ex}}  % end of proof
\newenvironment{proof}{\par\noindent{\bf Proof\ }}{\hfill\BlackBox\\[2mm]}
\newtheorem{example}{Example} 
\newtheorem{theorem}{Theorem}
 
\newtheorem{proposition}[theorem]{Proposition}

\newtheorem{definition}[theorem]{Definition}

\newcommand{\boldlambda}{{\boldsymbol \lambda}}

\newcommand{\Hone}{{\mathbf{1}}}
\newcommand{\Hid}{ {\mathit id}  }

\title{A general framework for defining and optimizing robustness}

\author{
  Alessandro Tibo\\
  Aalborg University, Institut for Datalogi\\
  \texttt{alessandro@cs.aau.dk}\\
  \And
  Manfred Jaeger \\
  Aalborg University, Institut for Datalogi\\
  \texttt{jaeger@cs.aau.dk}\\
  \And
  Kim G.~Larsen \\
  Aalborg University, Institut for Datalogi\\
  \texttt{kgl@cs.aau.dk}
}

\begin{document}
\maketitle

\begin{abstract}
  Robustness of neural networks has recently attracted a great amount of interest. The many investigations in this area
  lack a precise common foundation of robustness concepts. Therefore, in this paper, we propose a rigorous and flexible
  framework for defining different types of robustness properties for classifiers.
  Our robustness concept is based on postulates that
  robustness of a classifier should be considered as a property that is independent of accuracy, and that it should be
  defined in purely mathematical terms without reliance on algorithmic procedures for its measurement. 
  We develop a very general robustness framework that is applicable to any type of classification model, and that
  encompasses relevant robustness concepts for investigations ranging from safety against adversarial attacks to
  transferability of models to new domains. For two prototypical, distinct robustness objectives we then
  propose new learning approaches based on neural network co-training strategies for obtaining image classifiers
  optimized for these respective objectives.
\end{abstract}

\section{Introduction}\label{sec:intro}

Machine learning models have reached impressive levels of performance for solving various types of classification problems.
Famously, convolutional neural networks  \cite{fukushima1980neocognitron,lecun1989backpropagation},
are the state-of-the-art method for image classification (see, e.g., \cite{szegedy2017inception,xie2019self,kolesnikov2019large}).
However, the phenomenon of \emph{adversarial examples} \cite{szegedy2013intriguing,goodfellow2014explaining} has raised
serious concerns about the reliability of these machine learning solutions, especially in safety-critical applications.
Autonomous driving and counterfeit bill detection are  two intuitive examples where safety concerns play a fundamental role.
The vulnerability of machine learning models to adversarial examples is widely interpreted as a lack of
\emph{robustness} of these models, and a large amount of recent work investigates robustness properties, with a
strong focus on deep neural network models. Research in this area has been pursued both from the perspective
of formal verification with its traditional objective of rigorously proving safety properties of hard- and software
systems~\citep{ruan2018reachability,ehlers2017formal,dvijotham2018dual,bunel2018,Cisse-etal-2017,wicker2020probabilistic},
and from the perspective
of machine learning with its traditional objective of optimizing expected values of performance measures~\citep{madry2017towards,hein2017formal,elsayed2018large,sinha2019harnessing,su2019one}.

The original motivation derived from adversarial examples led to a focus on robustness against small
adversarial perturbations. 
As noted by \cite{gilmer2019adversarial}, this needs to be distinguished
from  robustness under larger scale distribution shifts.
The former is a particular
concern in order to guard against vulnerability to targeted attacks, whereas the latter is
important for applications where the characteristics of test cases may (eventually) differ from
the original training data.

In most works, the focus is on one of two separate issues: either \emph{assessing} the robustness properties
of a given model, or \emph{improving} the robustness properties of learned classifiers by developing new training
techniques. Works with a background in formal verification mostly focus on the first issue, while in the field
of machine learning the second  plays a larger role. However, in machine learning, too, the problem
of assessing robustness has received a lot of attention, especially in connection with designing adversarial
attacks \cite{goodfellow2014explaining,madry2017towards,hein2017formal,su2019one}.

The underlying concept of robustness used in different works
often only is implicit in a proposed loss function and/or the
experimental protocol used to evaluate robustness properties. As a result,  there does not yet
appear to be a common full understanding of what robustness actually is, and how it relates to other
properties of classification models. For example, \cite{tsipras2018robustness} have argued that the objectives
of robustness and accuracy are in conflict with each other, whereas \cite{stutz2019disentangling} come to
somewhat opposite conclusions: they argue that robustness and generalization capabilities are in conflict
only when robustness is designed to defend against ``off-manifold'' attacks, i.e., adversarial examples that
do not follow the data distribution, whereas they are consistent with each other in the scenario of
``on-manifold'' adversarial examples.
Similarly, there exist somewhat conflicting results with regard to the question whether robustness
against adversarial examples and against distribution shifts are in conflict with each other, or whether
these two objectives can be aligned~\citep{gilmer2019adversarial,rusak2020simple}.

Underlying these differences are not so much theoretical or empirical
discrepancies, as conceptual differences about what one wants to capture with robustness. In this paper we first
aim to put the analysis of robustness properties on a more solid foundation by  developing
a general framework for the specification
of different robustness concepts that captures most of the concepts previously used (implicitly or explicitly)
in the literature, and that helps to clarify their basic structural differences.
We propose a  flexible framework that captures in a coherent manner different robustness concepts
ranging from safety against adversarial attacks to the ability to adapt to distribution changes.
In this manner, we draw a direct link between the goals of learning robust and safe classifiers on
the one hand, and the traditional objectives of transfer learning on the other hand.

We then use our robustness definitions as the conceptual basis for developing robust learning
techniques in two protypical application settings: learning neural network classifiers that are
robust against small adversarial perturbations, and learning image classifiers that are robust
with respect to larger distribution shifts. 
Like most of previous research on robustness which focusses on neural network technology for image
classification problems, we, too, use readily available image benchmark datasets and
standard neural network architectures
to evaluate our methods. However, it is one of our core objectives that the conceptual framework is
generally applicable to all types of classification models, and the developed techniques are not specifically
tied to image data.

The key contributions of this paper are:
\begin{itemize}
\item A general analysis and specification framework for robustness concepts for classification models.
\item Development of robust learning techniques for two prototypical instantiations of the
  robustness framework covering two very different points in the robustness spectrum.
\end{itemize}

% The paper is organized as follows: Section 2 introduces  our general robustness framework with particular emphasis on several possible instantiations. Section 3 describes the methodology for learning under robustness objectives, with an emphasis on  two particular instantiations of our general robustness framework. Sections 4 and 5 detail key
% components of these instantiations, and our implementations.
% Sections 6,7 contain experimental results on real-world datasets. Finally we draw some conclusions in Section 8.

\section{A Robustness Framework}
\label{sec:shades}
We are considering classification problems given by an input space
$\mathcal{X}$, a label space $\mathcal{Y}= \{0,\ldots,K-1\}$, and a data distribution consisting of
an input distribution $P(X)$, and a conditional label distribution $P(Y|X)$. A classifier is any
mapping $f:\mathcal{X}\rightarrow \mathcal{Y}$. By as slight abuse of notation we also use $Y$ to
denote the \emph{true labeling functions} $Y: \mathcal{X}\rightarrow \mathcal{Y}$ defined by
$Y(x):=\argmax_y P(Y=y|X=x)$.
Most existing notions of robustness require that $\mathcal{X}$ is endowed with a metric.
Usually, ${\mathcal X} \subseteq \Rset^d$ for some $d\geq 1$, with a metric induced by
one of the standard norms on $\Rset^d$. In the following we will assume that ${\mathcal x}$ is
of this form.

Underlying our robustness framework are the  following two postulates:

\begin{description}
\item[P1.]  Robustness is orthogonal to accuracy (or generalization). In particular, a constant classifier with
  $f(x)=y$ for some fixed $y\in\mathcal{Y}$ and all $x\in \mathcal{X}$ always is maximally robust.
  This does not entail that robustness and accuracy are necessarily in conflict; only that they are
  separate, distinguishable objectives.
\item[P2.]   Robustness of a model $f$ is defined only in terms of $f$ itself, and the given classification problem.
  In particular, robustness is not dependent on specific (algorithmic) tools for assessing robustness.
\end{description}

P1 makes our concept of robustness different from robustness metrics used  e.g. by
\cite{tramer2019adversarial} or \cite{gilmer2019adversarial}, where robustness measures are directly
tied to accuracy. These differences notwithstanding, the measures we propose are consistent with these
earlier proposals in that they arguably extract their pure robustness components.
P2 implies that e.g. \emph{success rates} of particular adversarial example generators as used e.g.
in~\cite{su2018robustness},  or the CLEVER scoring algorithm of~\cite{weng2018evaluating} 
are not robustness measures in our sense per se, but rather tools for approximate evaluations of
proper robustness measures.

%\begin{itemize}
%\item[1.] 
%
%\item[2.] Robustness of a model $f$ is defined only in terms of $f$ itself, and the given classification problem.
%
%\item[3.] A definition of robustness should be independent of (algorithmic) tools for assessing robustness
%  of a concrete classifier. For example, the definition should not depend on  extraneous components such as
%  dimensionality reduction by a particular
%  algorithm (e.g. auto-encoder), or adversarial example generation techniques. 
%\end{itemize}

%\manfred{Here needs to be a comprehensive related work review about how other papers are positioned according
%  to these principles}
%Point 2. means that robustness should be defined only in terms of $f,P(X)$ and $P(Y|X)$. However, including the
%true label probabilities $P(Y|X)$ in the definition will usually lead to a conflation of robustness and accuracy
%properties. We therefore implement principle 1. by defining robustness only in terms of $f$ and $P(X)$. In
%the following, $B_{\epsilon}(x)$ denotes the $\epsilon$-ball around $x$. Unless otherwise states, this is
%with respect to the Euclidean metric.
%\color{black}

P2 can be refined to a hierarchy of robustness concepts reflecting their dependence
on different elements: we distinguish three types of robustness concepts, according to whether
robustness is defined in terms of
\begin{description}
\item[Type 1:] only the classifier $f$,
\item[Type 2:] the classifier $f$, and the input distribution $P(X)$
\item[Type 3:] $f, P(X)$, and the label distribution $P(Y|X)$  
\end{description}

% Type 1 leads to very strong robustness concepts capturing robustness under possibly significant
% changes of the input distribution (out-of-sample robustness) or different perturbation types~\cite{kang2019transfer}.
% However, the generality of the resulting robustness notions then will come at a cost for the
% accuracy of $f$ when data is generated by $P(X)$ and $P(Y|X)$. Robustness that is defined relative
% to the given $P(X)$ (and possibly $P(Y|X)$) leads to a focus on on-manifold robustness that
% is compatible with accuracy~\cite{stutz2019disentangling}, but less powerful with regard to
% inputs that are out-of-sample or malicious off-manifold attacks.

We now develop a  flexible framework  that can accommodate a wide
range of robustness concepts of all three types. In the following, 
${\cal A}({\cal X})$ denotes the $\sigma$-algebra of Borel sets on ${\cal X}$.

\begin{definition}
\label{def:robustnessfct}  
(Basic $Q$-Robustness Measure)
Let $f$ be a classifier and
\begin{equation}
  \label{eq:Qkernel}
  \begin{array}{llll}
    Q: & {\cal X}\times{\cal A}({\cal X})  & \rightarrow & [0,1] \\
    & (x,A) & \mapsto & Q(A|x) 
  \end{array}
\end{equation}
 a transition kernel on ${\cal X}$. 
The function
  \begin{equation}
    \label{eq:basicrobustnessfct}
    \begin{array}{rcl}
     \rho_{Q}^f:\ \  \mathcal{X} & \rightarrow & [0,1] \\
      x & \mapsto & Q( \{ x' : f(x')=f(x)  \} |x)
    \end{array}
  \end{equation}
  is called the \emph{basic robustness measure} of $f$ with respect to $Q$.
  % Similarly,
  %  \begin{equation}
  %   \label{eq:basicrobustnessfctlaware}
  %   \begin{array}{rcl}
  %    R_{Q,f,Y}:\ \  \mathcal{X}\times \Rset^+ & \rightarrow & [0,1] \\
  %     (x,\epsilon) & \mapsto & Q(f(X)=Y(x) | B_{\epsilon}(x))
  %   \end{array}
  % \end{equation}
  % is the \emph{label-aware basic robustness function} of $f$ with respect to $Q$ and $Y$.
\end{definition}

In conjunction with a distribution $P()$ on ${\cal X}$, a kernel $Q$ defines the probability measure
$\int_{\cal X} Q(\cdot|x)dP(x)$ on ${\cal X}$. We denote this measure as $Q\circ P$.

\begin{example}
  \label{ex:Qepsilon}
  Let $\epsilon >0$. We denote with $B_{\epsilon}(x)$ the open $\epsilon$-ball around $x$ (according
  to a chosen metric on ${\cal X}$). 
  Let  $Q_{\epsilon}( \cdot | x)$ be the uniform distribution on $B_{\epsilon}(x)$ ($x\in {\cal X}$).
  $\rho_{Q_{\epsilon}}^f(x)$ then is the probability that a perturbation of $x$ with a uniform random noise
  vector of length $\leq \epsilon$ is assigned by $f$ the same label as $x$. 
\end{example}

% We use $R_f$ as a generic notation for a basic robustness function of either form
% (\ref{eq:basicrobustnessfct}) or (\ref{eq:basicrobustnessfctlaware}).
The basic robustness measure measures the stability of the classifiers output
when examples $x$ are randomly perturbed according to the distribution $Q(\cdot|x)$.
Integrating $\rho_{Q}^f(x)$ over $x$  gives an overall robustness
score for $f$ relative to  $Q$-perturbations. 
However,
simply taking the integral over $\rho_{Q}^f$ is too crude to encode many important versions of
robustness. We therefore generalize the basic robustness measure by adding functions that allow
to extract specific features of $\rho_{Q}^f$.

\begin{definition}
  ($Q,H,G$-Robustness Function)
  Let $H:[0,1]\rightarrow \Rset^+$ be a monotone increasing function, and
  $G: {\cal X}\rightarrow \Rset^+$. Then
  \begin{equation}
    \label{eq:robustnessfct}
    \begin{array}{rcl}
      {\rho}_{Q,H,G}^f:\ \  \mathcal{X} & \rightarrow & \Rset^+ \\
      x & \mapsto & H(\rho_{Q}^f(x))G(x)   
    \end{array}
  \end{equation}
  is the \emph{robustness function} of $f$ defined by $Q,H$ and $G$.
  % If $R_f= R_{Q,f,Y}$ is label-aware,
  % then the robustness function also depends on $Y$.
\end{definition}

Integrating the robustness function then gives a robustness score for a classifier:

\begin{definition} \label{def:Qrobscore}
    ($Q,H,G$-Robustness Score)
	For a given $f$, $Q$,   $H$, and $G$ we define the $Q,H,G$-robustness score
	\begin{equation}
         \label{eq:Qrobscore}
		R_{Q,H,G}^f = \int_{\mathcal{X}} {\rho}^f_{Q,H,G}(x) dx.
	\end{equation}
\end{definition}

We note that for a constant classifier $f(x)=i\in{\cal Y}\ (x\in{\cal X})$ we have
$\rho^f_Q(x)=1$ for all $x$, and by the monotony of $H$ then also
$H(\rho^f_Q(x))$ is the maximal possible value for all $x$. This implies that for all
$Q,H,G$, 
$R_{Q,H,G}^f$ is maximal for a constant classifier, i.e., postulate P1 is satisfied.

The robustness score of Definition~\ref{def:Qrobscore} only relates to a single perturbation
model $Q$. Often, one will want to consider different forms of perturbations, possibly weighted by
a probability distribution.

\begin{definition} \label{def:QQrobscore}
    (${\cal Q},H,G$-Robustness Score)
  Let ${\cal Q}$ be a family of transition
  kernels $Q_{\boldlambda}$ parameterized by  $\boldlambda\in A\subseteq \Rset^k$.
  Let ${\cal Q}$ be equipped with a probability distribution defined by a density function
  $q_{\cal Q}$ on $A$. Let $H$ and $G$ as in Definition~\ref{def:Qrobscore}. Then we define the
  ${\cal Q}$-robustness score as
  \begin{equation}
         \label{eq:QQrobscore}
		R_{{\cal Q},H,G}^f = \int_A  R_{Q_{\boldlambda},H,G}^f   q_{\cal Q}(\boldlambda)  d\boldlambda .
	\end{equation}
\end{definition}

We next illustrate how the general robustness concepts introduced by
Definitions~\ref{def:robustnessfct}-\ref{def:QQrobscore} can be instantiated to obtain previously proposed
and novel robustness concepts of types 1-3.
 To simplify descriptions, from now on we assume that $\mathcal{X}$ is the $m$-dimensional unit hypercube,
  and we denote with $d =  \sqrt{m}$ its diameter.
  In this paper, we will consider for $H$ only  the two functions

  \begin{displaymath}
    \begin{array}{llll}
\label{eq:Hindicator}  
H(p) & = &\Hone(p) & := \indicator{p=1}, \ \ \mbox{and}\\
H(p) & = &\Hid(p) & :=  p
\end{array}
  \end{displaymath}

with $\indicator{\cdot}$ the indicator function. Thus, with $H=\Hone$ we are only interested in whether
the predicted label of $x$ will be almost surely preserved under the perturbation $Q(\cdot|x)$, which
leads to robustness concepts related to safety guarantees under adversarial perturbations, wheras
$H=\Hid$ leads to robustness concepts concerned with preservation of predictive accuracy.

\begin{example}
  \label{ex:geometry}
  (Geometry of decision regions; type 1.)
  For $\epsilon > 0$ let $Q_{\epsilon}$ be defined as in Example~\ref{ex:Qepsilon}. Define
  \begin{displaymath}
    {\cal Q}_{[0,d]} := \{ Q_{\epsilon} | \epsilon\in [0,d]\}  
  \end{displaymath}
  and let $q_{ {\cal Q}_0}$ be the uniform density on $[0,d]$. 
  Let $H=\Hone$,
  and $G(x)\equiv 1$. Under mild regularity conditions on
$f$\footnote{e.g. that $f(x)$ be defined as the argmax of a set of continuous discriminant
  functions  $f_y$ ($y\in{\cal Y}$)} we then have that
\begin{equation}
  \label{eq:indicator}
  \Hone(\rho^f_{Q_{\epsilon}}(x))=\indicator{\forall x'\in B_{\epsilon}(x): f(x')=f(x) },
\end{equation}\
i.e. the probabilistic statement $Q_{\epsilon}(\ldots)=1$ can be sharpened into a categorical ``for all'' statement.
Then
\begin{equation}
  R^f_{{\cal Q}_{[0,d]},\Hid,1}=  
  \int_0^d \int_{\cal X} \rho^f_{Q_{\epsilon},\Hone,1}(x) dx d\epsilon
  =  \int_{\cal X}\int_0^d\rho^f_{Q_{\epsilon},\Hone,1}(x) d\epsilon dx.
  \label{eq:exgeometry}
\end{equation}
The inner integral on the right is equal to the minimum distance from $x$ to a
decision boundary of $f$, i.e., the classifier margin at $x$.
Integrating over ${\cal X}$ then gives that $R^f_{{\cal Q}_{[0,d]},\Hid,1}$ is the average margin
for points $x\in {\cal X}$. This robustness measure only depends on the geometric complexity
of the decision regions of $f$.
\end{example}

\begin{example}
  \label{ex:succrates}
  (Margin curves; type 2.)
  
  Let ${\cal Q}_{[0,d]}$ be as in Example~\ref{ex:geometry}. Assume that $P(X)$ has a density function
  $p(x)$ relative to the uniform distribution, and let $G(x):=p(x)$. Then, under the same
  regularity assumptions as in Example~\ref{ex:geometry}, 
  $R_{Q_{\epsilon},\Hone,p}$
  is the probability that a point $x$ sampled according to $P(X)$ has no adversarial example at
  distance $\leq \epsilon$. Empirical estimates of this integral based on a number of test points
  $x_i$ correspond to robustness scores in terms of success rates of adversarial example generators
  (where the precision of the estimate then also is affected by the effectiveness of the generator).
  Seen as a
  function of $\epsilon$,  $R^f_{Q_{\epsilon},\Hone,p}$ defines the \emph{margin curves}
  of~\cite{gopfert2019adversarial}
  (up to a $1-\ldots$ inversion).
  Integrating  over $\epsilon\in[0,d]$ yields $R^f_{{\cal Q}_{[0,d]},\Hone,p}$, which measures robustness as the
  area under the margin curve.
\end{example}

% \begin{example}
%   \label{ex:onmanifoldworst}
%   (On-manifold, label-agnostic, worst case; type 2.)
%   Let $H,G$ as in Example~\ref{ex:succrates}, but let
%   $Q:=P$. We then obtain a robustness function and robustness score as in
%   Example~\ref{ex:succrates}, but candidate adversarial examples
%   must now be consistent with the input distribution $P(X)$, i.e. lie on the data
%   manifold~\cite{stutz2019disentangling}. The function $H$ only depends on the existence
%   of adversarial examples, not on their likelihood of being  generated by $P(X)$.
% \end{example}

% \begin{example}
%   \label{ex:onmanifoldaverage}
%   (On-manifold; type 2.)
%   For $\epsilon>0$ define
%   \begin{displaymath}
%     Q_{\epsilon,P}(\cdot|x):=P(\cdot | B_{\epsilon}(x))
%   \end{displaymath}
%   Let  $H=H_1$ and
%   $G(x):=p(x)$ as in Example~\ref{ex:succrates}. 
%   With $H_1$, the robustness function  no longer only depends on the existence of adversarial
%   examples, but on the probability that one would actually be generated by $P(X)$.
%   The robustness score $R^f_{Q_{\epsilon,P},H_1,G}$ 
%   simply becomes the probability that two  points  $x,x'$
%   randomly chosen according to $P(X)$ 
%   at a distance $<\epsilon$ of each other are assigned the same label.
%   \manfred{Depending on what we need later, must extend this example with a probability distribution
%   over the $\epsilon$}
% \end{example}

\begin{example}
  (label aware; type 3.)
  Label-aware (type 3) robustness functions can be obtained by conditioning the perturbation
  distribution $Q_{\epsilon}(\cdot | x)$  also on the label of $x$. Thus we let
  \begin{displaymath}
    Q_{\epsilon,Y}(\cdot | x):= Q_{\epsilon}(\cdot |x)| \{ x': Y(x')=Y(x) \},
  \end{displaymath}
  where we use the general notation $P|A$ to denote a probability distribution $P$ conditioned on
  an event $A\in {\cal A}({\cal X})$. 
  With $H$ and $G$ as in Example~\ref{ex:succrates}, one now
  obtains the robustness concept implicitly (through the definition of
  the \emph{expected adversarial loss} function) used in~\cite{tsipras2018robustness}
  and adapted in~\cite{gopfert2019adversarial}.
\end{example}

In all the preceding examples the data distribution $P(X)$ was only used to weigh the contribution
of points $x\in {\cal X}$ to the overall robustness score. It was not used in the definition of
the perturbation model ${\cal Q}$.
The resulting robustness models then capture scenarios where datapoints generated by $Q\circ P$ can
be  off-manifold  in the sense of~\cite{stutz2019disentangling}. For the more conservative setting of
robustness with respect to 
on-manifold adversarial examples, one can also use $P()$ in the definition of $Q$ by letting
  \begin{displaymath}
    Q_{\epsilon,P}(\cdot|x):=P | B_{\epsilon}(x).
  \end{displaymath}
  In all preceding examples, we can substitute $Q_{\epsilon,P}$ for $Q_{\epsilon}$, and thereby obtain
  on-manifold versions of the given robustness concepts. The question of whether an off-manifold or on-manifold
  concept of robustness is more appropriate is quite subtle. The former is more relevant when robustness
  against malicious attacks or substantial distribution changes is required; the latter may be sufficient
  when only robustness against distribution changes is required that are induced e.g. by changing
  observation biases, but do not modify the underlying nature of the sampling population.

None of the preceding definitions and examples takes the accuracy of the model $f$ into account, i.e., we
never considered whether $f(x)=Y(x)$. This is in line with our stipulation that robustness as a
concept is orthogonal to accuracy. Obviously, one  ultimately wants robustness to lead to accuracy
gains under some conditions. In order to link robustness to accuracy, we first give and
explicit definition for an assumption that usually  implicitly made for
adversarial perturbations or common corruptions. 

\begin{definition}
  Let $P$ be a distribution on ${\cal X}$ with density function $p()$,
  and $\epsilon>0$. A kernel $Q$ is $(P,\epsilon)$-\emph{label preserving} if
  \begin{displaymath}
    \int Q(\{x': Y(x') = Y(x) | x\})p(x)dx \geq 1-\epsilon
  \end{displaymath}
\end{definition}

Viewing $Q$ as a formal model for adversarial perturbations or corruption mechanisms,
standard assumptions in the literature would translate into an assumption that
$Q$ is $(P,0)$-label preserving. 
We can  now formulate the following trivial but noteworthy relationship between robustness and
accuracy:

\begin{proposition}
  \label{prop:accpreservation}
  Let $f$ have accuracy $\alpha$ on data sampled according to the distribution $P$. Let $Q$ be
  $(P,\epsilon)$-\emph{label preserving}, and assume that $R^f_{Q,\Hid,p}= 1-\delta$.
  Then $f$ has accuracy on data sampled according to $Q\circ P$ of at least $\alpha-\epsilon-\delta$.
\end{proposition}

\begin{proof}
  Let $A=\{x | f(x)=Y(x)\}$ be the set of inputs on which $f$ makes correct predictions. Then
  $P(A)=\alpha$. Using that $Q(\cdot|\cdot)$ is non-negative, and therefore the integral of
  $Q(\cdot|x)$ over the full space is always greater or equal the partial integral over $A$, we
  obtain:
  \begin{multline}
    (Q\circ P)(A)=\int Q(A|x)p(x)dx\geq \int_A Q(A|x)p(x)dx\\
    \geq  \int_A Q( \{ x':Y(x')=Y(x)  \}\cap\{ x': f(x')=f(x)\}   |x)p(x)dx \\
    \geq  \int_A (1-Q( \{ x':Y(x')\neq Y(x)  \}|x)-Q(\{ x': f(x')\neq f(x)\}   |x))p(x)dx \\
    =\alpha - \int_AQ( \{ x':Y(x')\neq Y(x)  \}|x)p(x)dx - \int_AQ(\{ x': f(x')\neq f(x)\}   |x)p(x)dx \\
    \geq \alpha - \int Q( \{ x':Y(x')\neq Y(x)  \}|x)p(x)dx - \int Q(\{ x': f(x')\neq f(x)\}   |x)p(x)dx
    \geq \alpha-\epsilon-\delta.
  \end{multline}
\end{proof}
% {\color{gray}

% The following theorem links robustness of a model to accuracy guarantees under slight changes of
% the input distribution. In the following we denote with $\WK(P(X),P'(X))$ the Wasserstein-Kantorovich
% distance between input distributions $P$ and $P'$. 

% \begin{theorem}
%  Let $P,f$ be such that $f$ has accuracy $P(f(X)=Y)=\alpha$.
%  Then the accuracy of $f$ under a different input distribution $P'$ is lower bounded by
%  \begin{displaymath}
%    P'(f(X)=Y)\geq \alpha-\Delta(\alpha)
%  \end{displaymath}
%  where $\Delta(\alpha)$ is defined by the relationship
%  \begin{equation}\label{eq:ws}
%    \WK(P,P')=\int_0^{M_{P,f}^{-1}(\Delta(\alpha))}\epsilon M_{P,f}'(\epsilon)\, d\epsilon,
%  \end{equation}
%  where $M_{P,f}'(\epsilon)$ is the density function of the margin distribution. 
% \end{theorem}
% }
%
%\manfred{This theorem is meant to exemplify what one should get out of well-defined robustness concepts: the ability
%  to derive guarantees for certain properties of (primary) interest (robustness per se not being a property
%  of primary interest, but an auxiliary property). I am not sure how strong this particular theorem is as an application
%  of the general framework we here set up. The theorem could be derived in a much more direct manner than starting
%  with the general robustness function of Definition~\ref{def:robustnessfct}, and the statement of the theorem is perhaps
%neither very surprising,  nor does it provide a very strong or useful analytic tool.}

\section{Optimizing a Given Robustness Measure}\label{sec:methodology}

In this section we consider the case where a target robustness version is fully specified in terms
of specific ${\cal Q}, H, G$. An important case of this nature is the one described in Example~\ref{ex:succrates},
which has been quite extensively considered for neural network classifiers.
We therefore  also focus in this section on neural network classifiers  $f$.
We assume that classification is performed
by taking the $\argmax$ over an output layer produced by a softmax function. We denote with
$\hat{f}(x)\in[0,1]^K$ the network output for input $x$; $\hat{f}(x)[j]$ denotes the $j$th component of
$\hat{f}(x)$.

\subsection{Robust co-training for  $R^{\lowercase{f}}_{{\cal Q}_{[0,{\lowercase{d}}]},{\lowercase{\Hone}},{\lowercase{p}}}$}\label{sec:ouradvgenerator}

Existing approaches for training $f$ with high $R^f_{Q_{\epsilon},\Hone,p}$ scores include data augmentation
with adversarial examples~\cite{madry2017towards} and the use of customized regularization terms in the loss
function~\citep{hein2017formal}. The construction of adversarial examples is computationally expensive,
especially when it has to be repeated in each iteration of the training loop for $f$. We therefore
propose a hybrid approach that employs computationally cheap adversarial example generation to
define a regularization term in the loss function for $f$.

\begin{figure}[b]
\centering
\centerline{\includegraphics[width=0.3\columnwidth]{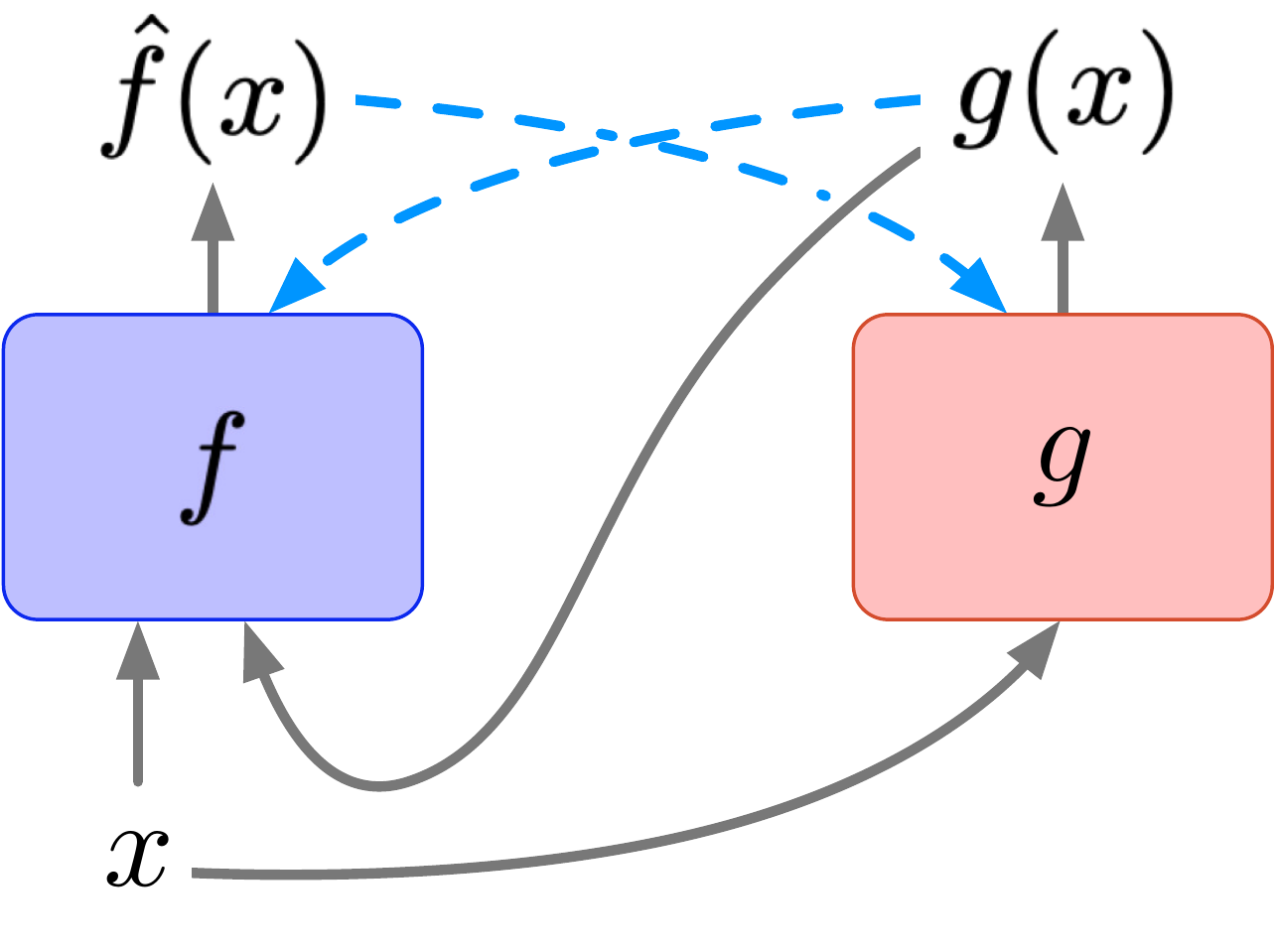}}
\caption{Co-training of classifier $f$ and adversarial generator. Solid arrows represent
input/output relations; dashed arrows represent dependence of the loss function.}
\label{fig:robustness-framework}
\end{figure}

Concretely, our approach consists of co-training the classifier $f$, and a generator $g$ that
maps inputs $x$ to approximate adversarial examples $g(x)$. Given a (current) classifier $f$, our
loss function for $g$ for a training example $x$ is
\begin{equation}\label{eq:proposed:reg}
{\cal L}_g(x,\theta_g) =  \frac{1}{2} \| g(x;\theta_g ) - x \|^2 +
\max(\hat{f}(g(x;\theta_G))[i] - \hat{f}(g(x;\theta_G))[j],\ 0) ,  
\end{equation}
where $\theta_g$ are the trainable parameters of $g$, $i=f(x)$ is the label associated with $x$ by $f$,
and $j\in{\cal Y}$  is the label with the second highest value in $\hat{f}(g(x;\theta_G))$.
The first term of the loss function enforces that $g(x)$ is close to $x$.
The second term is equal to zero if $f(g(x))\neq f(x)$, i.e., $g(x)$ is adversarial in the sense
that it is labeled differently from $x$.
% while the second term tries to generate $g(x)$ in  such a way that $f(g(x))\neq f(x)$.

Given a (current) generator $g$, the classifier $f$ is trained using the following loss function
for a labeled training example $(x,y)$: 
\begin{equation}
  \label{eq:final:loss:ex:1}
  {\cal L}_f(x,\theta_f) = -\log( \hat{f}(x)[y]  ) +
  \lambda\frac{CE(\hat{f}(x),\hat{f}(g(x))}{1+ \parallel x-g(x)\parallel^2  }.
\end{equation}
The first term here is the standard log loss. The second term is our robustness regularizer.
It is essentially inversely proportional to how successful the generator $g$ is: the loss increases
if $g$ finds close adversarial examples, i.e. $\parallel x-g(x)\parallel^2$ is small, and it
increases when the output of $f$ at $g(x)$ is very different from the output of $f$ at $x$, as
measured by $CE(\hat{f}(x),\hat{f}(g(x))$.

The models $f$ and $g$ are co-trained by an alternating stochastic gradient descent as
shown in Algorithm~\ref{alg:reg:training}.

% After training the generator $g$ for a fixed $f$, it is used to generate adversarial examples
% for all training examples $x^{(i)}$, which then determine the loss function in the next
% iteration of training $f$. 
% A detailed description of the co-training of $f$ and $g$ is given in Algorithm~\ref{alg:reg:training}.

% As described in the previous section, the training for robustness loss 
% requires an (adversarial) sample generator $g$ (cf. Figure~\ref{fig:robustness-framework}). In this
% section we describe our design for that generator.

\begin{algorithm}
\SetAlgoLined
 {\bfseries Input:} $D$, $f$, $g$, batchsize\;
 Randomly initialize $f$ and $g$\;
 \While{not converged}{
  $\{I^{(f)}_1,\ldots,I^{(f)}_K \}$ = shuffle(D)  //create $K$ mini-batches according to batchsize\;
  $\{I^{(g)}_1,\ldots,I^{(g)}_K \}$ = shuffle(D)\;
  \For{$i=1$ {\bfseries to} $K$}{
  	$x'_j = g(x_j) \quad \forall \ j = 1,\ldots, |I^{(f)}_i| $\;
  	Perform SGD update on $f$ using augmented mini-batch $\{x_j,x'_j \ | \ j=1, \ldots, |I_i^{(f)}| \}$\;
  	Perform SGD update on $g$ using inputs from $I_i^{(g)}$\;
  }
 }
 \caption{Regularization Training for $R^f_{{\cal Q}_{[0,d]},\Hone,p}$}\label{alg:reg:training}
\end{algorithm}

\subsection{Robustness Evaluation}\label{sec:robeval}

Having learned a classifier $f$, we want to evaluate its $R^f_{{\cal Q}_{[0,d]},\Hone,p}$ score. To obtain
an empirical estimate based on a set of test points $x^{(i)}$, one would ideally determine precisely
the classifier margin at each $x^{(i)}$. Since this is infeasible, we resort as usual to
the distance between $x^{(i)}$ and an adversarial example for $x^{(i)}$ as an upper approximation
of the margin. For this purpose we could re-use the generator model $g$ constructed in our
robust training approach. However, this would lead to an unfair advantage of our models in comparsion
to models learned using other robust training approaches. 
% precisely, one would have to 
% For the particular Example~\ref{ex:succrates}
% we can base the scoring on the same generators $g$ as described in
% Sections~\ref{sec:ouradvgenerator} in the context of training.
% However, when scoring robustness w.r.t. $\tilde{R}^f_{{\cal Q}_0, H_0, G}$ of model $f$ learned by different techniques, then
% re-using our adversarial generator $g$ used in training also for scoring purposes would bias the results.
For scoring, we therefore construct adversarial examples $x^*$ using an ensemble approach, 
where we consider a set of state-of-the-art strategies for generating adversarial examples, FGSM~\cite{goodfellow2014explaining}, PGD~\cite{madry2017towards}, and BG~\cite{hein2017formal}. Starting from
a trained model $f$, for each $x$ in the test set, we construct the set $Adv(x)=\{ x_{FGSM}, x_{PDG}, x_{BG}\}$, where its elements are generated according to the strategies mentioned above. Then, for each 
$x' \in Adv(x)$ we consider $\hat{x} = x + \epsilon_{x'} (x' - x)$, where $\epsilon$ 
is the smallest positive real number such that $f(x + \epsilon (x' - x)) \neq f(x)\}$. 
In practise $\epsilon$ is estimated by a binary search under the assumption that there is only 
one decision boundary on the line between $x$ and $x'$.
Finally, we select from among $\hat{x}_{FGSM}$, $\hat{x}_{PDG}$, $\hat{x}_{BG}$ the one that lies closest to
$x$ as the adversarial example $x^*$.

% \begin{itemize}
% \item 
%  % Note that all those methods do not generate adversarial example according to specific different labels;
%   \item for a trained model $f$ we generate all the adversarial examples according to the methods listed above, i.e. for each $x$ in the test set we generate a set :
%   \item for each $x' \in Adv(x)$ we determine
%   \begin{equation}\label{eq:epsilon} 
%     \epsilon_{x'} = \min\{\epsilon: f(x + \epsilon (x' - x)) \neq f(x)\}
%       \end{equation}
%       and let $\hat{x} = x + \epsilon_{x'} (x' - x) $. In practice, $\hat{x}$ is estimated b.
%     % \item we choose then $\hat{x}$ as an approximation of $\epsilon^*(x) \simeq \|x - \hat{x} \|$.  Figure \ref{fig:rob:eval:cartoon} depicts a cartoon illustrating the estimation of the margin for a point $x$ ;
%     \item we select from among $\hat{x}_{FGSM}$, $\hat{x}_{PDG}$, $\hat{x}_{BG}$ the one that lies closest to
%       $x$ as the adversarial example $x^*$.
% \end{itemize}

\begin{figure}[ht]
\centering
\centerline{ \includegraphics[width=0.4\columnwidth]{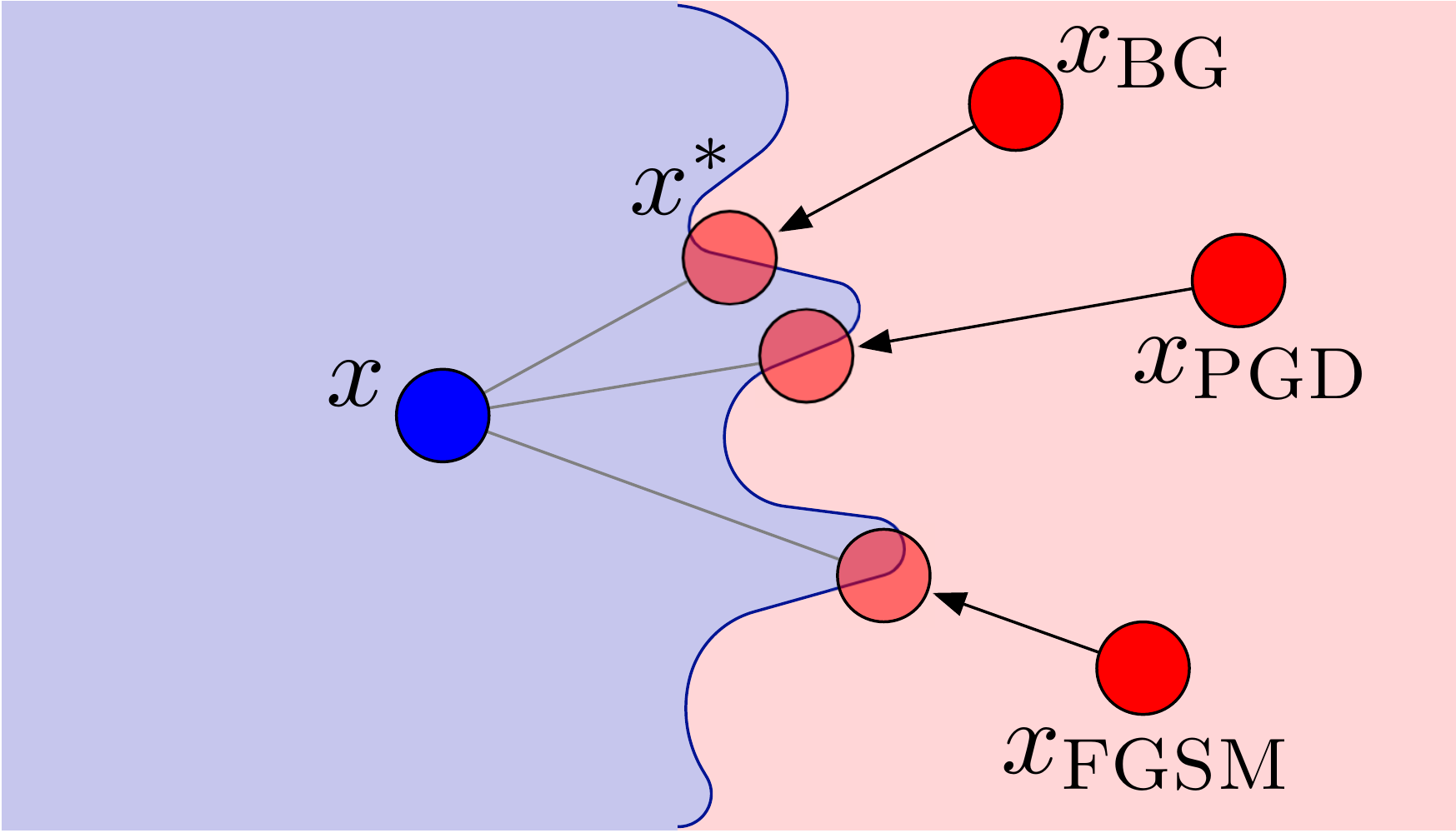}}
\caption{Ensemble approach for adversarial example construction. Shown are the decision boundary, 
  a test point $x$, the adversarial examples generated for $x$ by the methods BG, PGD, FGSM, their
  associated points $\hat{x}$, and the finally selected $x^*$.}
  % The blue thick line represents the local decision boundary for $x$. $x_{FGSM}$, $x_{PGD}$, and $x_{BG}$ are adversarial examples generated from $x$. By using binary search all the adversarial examples are moved to the decision boundary on the direction of $x$. }
\label{fig:rob:eval:cartoon}
\end{figure}

\subsection{Experimental results}\label{sec:experiments}
% TO MOVE TO THE APPENDIX
%\begin{figure*}[htp]
%\vskip 0.2in
%\begin{center}
%\centerline{\includegraphics[width=\textwidth]{images/minst_adv_examples.pdf}}
%\caption{TODO: This should perhaps go in appendix.}
%\label{fig:mnist:adv:examples}
%\end{center}
%\vskip -0.2in
%\end{figure*}

% TO MOVE TO THE APPENDIX
%\begin{figure*}[htp]
%\vskip 0.2in
%\begin{center}
%\centerline{\includegraphics[width=\textwidth]{images/fminst_adv_examples.pdf}}
%\caption{TODO: This should perhaps go in appendix.}
%\label{fig:fmnist:adv:examples}
%\end{center}
%\vskip -0.2in
%\end{figure*}

\begin{figure*}[t]
\centering
\centerline{\includegraphics[width=\textwidth]{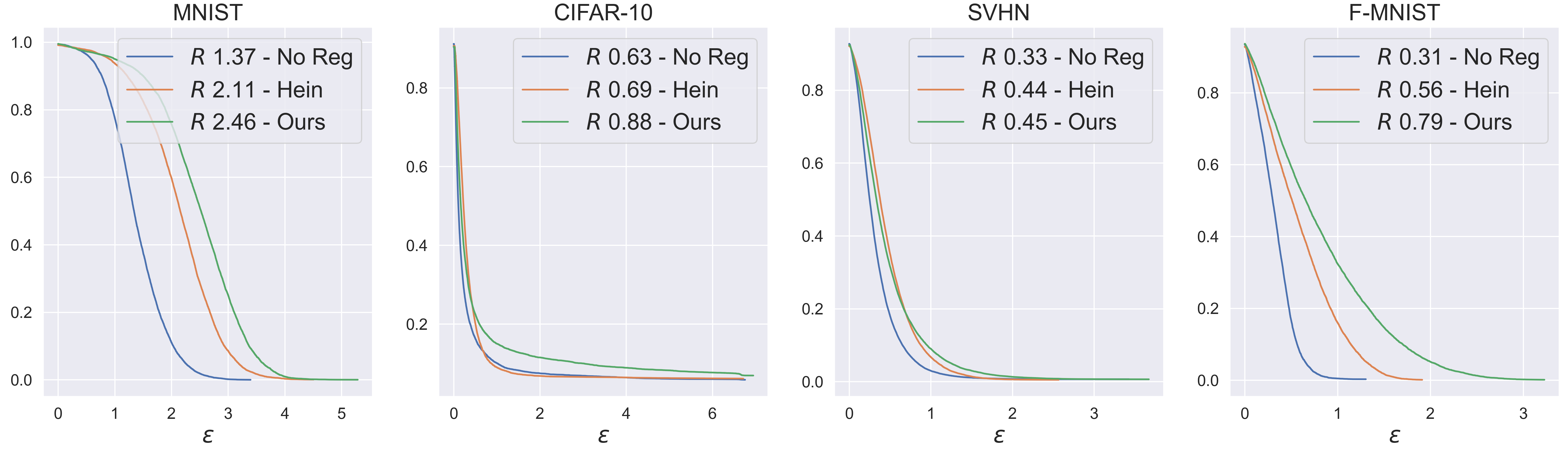}}
\caption{Qualitative and quantitative robustness results for the MNIST, CIFAR-10, SVHN, and F-MNIST. The legends report the area under the curves as quantitative measure of robustness. Higher areas corresponds to more robust models.}
\label{fig:gopfert}
\end{figure*}

We evaluated the proposed training approach for Example~\ref{ex:succrates} on 4 datasets: MNIST~\cite{lecun1998gradient}, CIFAR-10~\cite{krizhevsky2009learning}, SVHN~\cite{netzer2011reading}, and F(ashion)-MNIST~\cite{xiao2017fashion}. 
MNIST contains $28\times 28$ pixel gray scale handwritten digits uniformly distributed over $10$ classes, divided in $60,000$ and $10,000$ examples for training and test respectively. SVHN contains $32 \times 32$ pixel RGB images of street view house numbers uniformly distributed over $10$ classes, divided in $73,257$ and $26,032$ examples for training and test respectively. CIFAR-10 contains $32 \times 32$ pixel RGB natural images uniformly distributed over $10$ classes, divided in $50,000$ and $10,000$ examples for training and test respectively. Finally F-MNIST is a dataset of Zalando's article $28\times 28$ images, consisting of a training set of $60,000$ examples and a test set of $10,000$ examples, uniformly distributed within $10$ classes.

For each dataset we compared 3 neural networks $f$ which share the same structure: a model trained without any robustness regularization, a model trained with the regularization penalty from~\cite{hein2017formal} %showed in Equation~\ref{eq:hein:reg}
, and a model trained according to our proposed regularization penalty defined in Equation \ref{eq:final:loss:ex:1}, in conjunction with the adversarial generator described in Section~\ref{sec:ouradvgenerator}.

For MNIST, SVHN, and F-MNIST we used for  $f$ the same convolutional neural network architecture.
For CIFAR-10, we used instead the state-of-the-art model ResNet 20 described in \cite{he2016deep}.
The generator $g$ is another convolutional neural network, whose structure is the same for all the datasets (with the exception of the filters of the last layer that depend whether the input picture is grayscale or RGB). The  implementation details of $f$ and $g$ are described in the supplementary material. 
We  first trained the models without any robustness regularization, obtaining  accuracies as
shown in the first column of Table~\ref{tab:accuracies}. We then trained using the robustness regularizer of
\cite{hein2017formal} (in the following referred to as {\sc Hein} method) and our approach using different
values of the $\lambda$ parameter that trades off accuracy vs. robustness (cf.
Equation~\ref{eq:final:loss:ex:1}; {\sc Hein}
has a corresponding parameter).
We report in the following the results that were obtained with the  $\lambda$ parameter value that matched
most closely the learning without robustness regularization in terms of test set accuracy.
The second and third column of Table~\ref{tab:accuracies} show the obtained accuracies.
Note that here we are not so much interested in improving the state of the art in terms of accuracy,
but in improving robustness of reasonably accurate models.
Therefore, we  compare the robustness scores on models that have similar accuracies.

\begin{table}[ht]
\caption{Classification accuracies for the three setups for all the datasets.}
\label{tab:accuracies}
\begin{center}
\begin{small}
\begin{sc}
\begin{tabular}{lcccr}
\toprule
Data set & No Reg. & Hein & Ours \\
\midrule
MNIST     & 99.42\% & 99.12\% & 99.49\% \\
CIFAR-10 & 91.15\% & 90.61\% & 90.38\% \\
SVHN    & 92.64\% & 92.08\% & 92.33\% \\
F-MINST    & 93.47\% & 92.80\% & 93.57\%  \\
\bottomrule
\end{tabular}
\end{sc}
\end{small}
\end{center}
\end{table}

% \input{../tables/adv_examples_percentages.tex}

% $\lambda=0.001$  & $\lambda=0.5$ \\
% $\lambda=0.0005$  & $\lambda=0.25$ \\
% $\lambda=0.01$  & $\lambda=0.5$ \\
% $\lambda=0.002$  & $\lambda=0.5$ \\

For evaluating the robustness of the models, we used the adversarial example construction described in Section \ref{sec:robeval}.
In Figure \ref{fig:gopfert} we plot the $R^f_{Q_{\epsilon},\Hone,p}$ scores against $\epsilon$-values.
These are essentially the  marginal curves defined by \cite{gopfert2019adversarial}. The areas under these
curves, i.e. our $R^f_{{\cal Q}_{[0,d]},\Hone,p}$ score are reported in the legends, denoted with $R$. 
The results suggest that our proposed method allows to learn more robust models while keeping the accuracy almost unchanged. We also evaluated the effectiveness of the three methods BG, FGSM and PGD 
for generating adversarial examples. Due to the space limitation this analysis is reported 
in the supplementary material.

\section{Unknown, large distribution shifts}
In contrast to the adversarial robustness considered in the previous section,
robustness under distribution shift or common corruptions~\citep{rusak2020simple}
addresses the situation where new test cases may differ substantially from the
original training data, which in our framework means that the conditional distribution
$Q(\cdot | x)$ is not concentrated in an $\epsilon$-neighborhood of $x$.
Without such a locality assumption on $Q$, it is then very difficult to encapsulate
a sufficiently general robustness requirement by the specification of a specific
family ${\cal Q}$ in analogy to ${\cal Q}_{[0,d]}$ of the previous section.

Recent investigations on distribution shift robustness are mostly based on image
benchmark datasets, where an existing collection of labeled images is extended
with additional test sets consisting of original test images that are modified by
different types of corruptions, such as blurring, rotations, or change in brightness.
Corrupted test sets of this kind have been introduced for the MNIST character recognition
set~\citep{mu2019mnist}, and for ImageNet~\citep{hendrycks2018benchmarking}. The goal
then is to train a model on original, uncorrupted, training data, such that it maintains
high accuracy on the corrupted test images. 
Trying to obtain robustness in this setting by robust training
techniques developed for adversarial robustness has met with limited
success~\citep{mu2019mnist,rusak2020simple}. Current state-of-the-art approaches for
corruption robustness are based on training set augmentation, where additional
training instances are generated from a mixture of pre-defined image perturbations~\citep{hendrycks2019augmix},
or added Gaussian noise sampled from a noise distribution optimized by co-training the
classifier and the noise generator~\citep{rusak2020simple}. 

Most existing approaches are quite tightly linked to image data and neural network classifiers.
We propose a  general method that is not based on any specific properties of image data, and that
can be combined with any type of classifier. Seeing that in this context a preservation of
accuracy is the goal (not resistance to adversarial attacks), we are aiming for robustness as
measured by  $R^f_{Q,\Hid,p}$ (cf. Proposition~\ref{prop:accpreservation}). However, the transition kernel
$Q$ modeling the shift in the data distribution is not known a-priori. Our general idea is
to
\begin{description}
\item[1.] learn a classifier $f$ from labeled training examples $(x,y)$ sampled from $P(X)P(Y|X)$;
\item[2.] from original data points $x$, and new unlabeled examples $\tilde{x}$ sampled from $Q\circ P$, learn  a
  \emph{transfer mapping} $T$ that approximates an inversion of the transition
  $x\mapsto  \tilde{x} \sim Q( | x)$;
\item[3.] classify new test instances $\tilde{x}$ as $f(T(\tilde{x}))$. 
\end{description}

The learning of $f$ and $T$ are completely separate, which allows the method to be used in conjunction with 
any type of classifier. Since in step 2. we are using examples from  $Q\circ P$, this approach
does not follow the strict condition that all training has to be done on the original, uncorrupted
training set. However, the approach reflects a realistic application scenario, where one 
modifies an original classifier $f$ to a classifier $f\circ T$ adapted to a changing data distribution.
It is important to note that for this adaptation we neither require labeled examples $(\tilde{x},y)$, nor
do we use  supervision in the form of given pairs $(x,\tilde{x})$ for which $\tilde{x}\sim Q(\cdot| x)$. 
In effect, our approach can be seen as a type of \emph{transductive  transfer learning}
in the sense of~\citep{pan2009survey}.

% We first trained a classifier $f$ on the original MNIST 
% training set. Note that $f$ can be any general machine learning 
% classifier, e.g. SVM, Neural Networks.
Our approach for constructing a transfer model $T$  follows
a paired auto-encoder architecture that in similar form has previously
been used for the quite different application of deep fake generation~\citep{DBLP:journals/corr/abs-1909-11573}.
We jointly train a pair of autoencoders to reconstruct
samples of original points $x$, and corrupted points $\tilde{x}$. 
% a sample
% from MNIST and C-MNIST, respectively.
The two autoencoderes share
the same encoder $E$, but have different decoders, 
$D_1$ and $D_2$, as depicted
in Figure~\ref{fig:deepfake} (left) for the case of clean and corrupted MNIST digits.
% Altough the context
% is completly different, this idea was used before to swap
% faces of different people within a video
% (see e.g.~\cite{DBLP:journals/corr/abs-1909-11573}).
The transfer function $T$ then is defined as
$T(\tilde{x})=D_1(E(\tilde{x}))$
% Before classifying an example with $f$, 
% we first encode the potential corrupted
% picture through $E$, and then decode it by applying $D_1$
(see Figure~\ref{fig:deepfake} (right)).

\begin{table*}[ht]
  \caption{Classification accuracies for NN and SVM models using different robust
    robust training techniques. 
For both models, the first row reprents the average accuracy over the 
15 different corruption types. The second row shows the accuracy on the 
original MNIST test set. The standard deviation are estimated by training
the models (including the autoencoders) with 5 different weight initializations.
}
\label{tab:cmnist:svm:nn}
\begin{center}
\begin{small}
\begin{sc}
\setlength\tabcolsep{3.0pt}
\begin{tabular}{cccccccccc}
\toprule
\multirow{2}{*}{Model} & \multirow{2}{*}{Plain} & \multirow{2}{*}{Gauss} & \multicolumn{4}{c}{Transfer-Separate} & \multicolumn{3}{c}{Transfer-Joint} \\
 \cmidrule(lr){4-7}
 \cmidrule(lr){8-10}
 & & & 0 & 128 & 1024 & 30k & 128 & 1024 & 30k \\
 \cmidrule(lr){1-1}
 \cmidrule(lr){2-2}
 \cmidrule(lr){3-3}
 \cmidrule(lr){4-4}
 \cmidrule(lr){5-5}
 \cmidrule(lr){6-6}
 \cmidrule(lr){7-7}
 \cmidrule(lr){8-8}
 \cmidrule(lr){9-9}
 \cmidrule(lr){10-10}

\multirow{2}{*}{NN} & 86.3$\pm$0.52 & 88.5$\pm$0.38 & 86.4$\pm$0.87 & 88.5$\pm$0.44 & 91.1$\pm$0.16 & 92.8$\pm$0.09 & 88.5$\pm$0.17 & 91.5$\pm$0.05 & 92.1$\pm$0.15 \\
& 99.1$\pm$0.03 & 99.3$\pm$0.05 & 99.1$\pm$0.03 & 99.0$\pm$0.02 & 99.0$\pm$0.04 & 99.1$\pm$0.01 & 99.0$\pm$0.01 & 99.1$\pm$0.05 & 99.1$\pm$0.01\\ 
\midrule
\multirow{2}{*}{SVM} & 69.7 & 70.6$\pm$0.07 & 70.7$\pm$0.74 & 81.0$\pm$0.89  & 86.2$\pm$0.41 & 87.8$\pm$0.30 & 83.1$\pm$0.67 & 85.4$\pm$0.41 & 85.4$\pm$0.40 \\
& 98.4 & 97.9$\pm$0.02 & 98.3$\pm$0.06 & 98.3$\pm$0.06 & 98.2$\pm$0.02 & 98.3$\pm$0.03 & 98.2$\pm$0.07 & 98.3$\pm$0.05 & 98.3$\pm$0.03 \\

% \multirow{2}{*}{NN} & 86.30$\pm$0.52 & 88.54$\pm$0.38 & 86.37$\pm$0.87 & 88.47$\pm$0.44 & 91.06$\pm$0.16 & 92.75$\pm$0.09 & 88.50$\pm$0.17 & 91.48$\pm$0.05 & 92.07$\pm$0.15 \\
% & 99.42$\pm$0.03 & 99.26$\pm$0.05 & 99.05$\pm$0.03 & 99.01$\pm$0.02 & 98.97$\pm$0.04 & 99.07$\pm$0.01 & 98.95$\pm$0.01 & 99.10$\pm$0.05 & 99.12$\pm$0.01\\ 
% \midrule
% \multirow{2}{*}{SVM} & 69.69 & 70.57$\pm$0.07 & 70.71$\pm$0.74 & 81.00$\pm$0.89  & 86.16$\pm$0.41 & 87.76$\pm$0.30 & 83.13$\pm$0.67 & 85.44$\pm$0.41 & 85.37$\pm$0.40 \\
% & 98.37 & 97.86$\pm$0.02 & 98.30$\pm$0.06 & 98.29$\pm$0.06 & 98.21$\pm$0.02 & 98.29$\pm$0.03 & 98.21$\pm$0.07 & 98.30$\pm$0.05 & 98.31$\pm$0.03 \\
\bottomrule
\end{tabular}

% \begin{tabular}{lclc}
% \toprule
% Model & Accuracy & Model & Accuracy \\
% \cmidrule(lr){1-2}
% \cmidrule(lr){3-4}
% \textsc{NN-Plain} & 86.30$\pm$0.52 & \textsc{SVM-Plain} & 69.69 \\ 
% \textsc{NN-Gauss} & 88.54$\pm$0.38 & \textsc{SVM-Gauss} & xx.xx \\ 
% \textsc{NN-128} & 88.47$\pm$0.44 & \textsc{SVM-128} & 81.00$\pm$0.89 \\ 
% \textsc{NN-1024} & 91.06$\pm$0.16 & \textsc{SVM-1024} & 86.16$\pm$0.41 \\ 
% \textsc{NN-30}$k$ & 92.75$\pm$0.09 & \textsc{SVM-30}$k$ & 87.76$\pm$0.30 \\ 
% \textsc{NN-A-128} & 88.50$\pm$0.17 & \textsc{SVM-A-128} & 83.13$\pm$0.67 \\ 
% \textsc{NN-A-1024} & 91.48$\pm$0.05 & \textsc{SVM-A-1024} & 85.44$\pm$0.41 \\ 
% \textsc{NN-A-30}$k$ & 92.07$\pm$0.15 & \textsc{SVM-A-30}$k$ & 85.37$\pm$0.40 \\ 
% \bottomrule
% \end{tabular}
\end{sc}
\end{small}
\end{center}
\end{table*}

% DATASET
We evaluate our method using the
C-MNIST~\citep{mu2019mnist} suite of 15 corruptions
% ~\footnote{brightness, canny edges, dotted line, fog, glass blur, impulse noise, motion blur, rotate, scale, shear, shot noise, spatter, stripe, translate, zigzag.}
applied to the MNIST training and test sets.
% It is intended as a set of test cases for robust MNIST classifiers, 
% and no examples from this set
% should be used during training.
% METHOD
We consider two different classifers:
a Neural Network (NN), and a support vector machine (SVM) with RBF kernel. 
For the NN we used the same architecture as used
by~\cite{mu2019mnist}. The details regarding the structure
and the training of $f$, $E$, $D_1$, and $D_2$ are reported 
in the supplementary material. 
We use  two  different evaluation setups:
\begin{itemize}
\item \textsc{Transfer-Separate}: we train and test transfer models with data of one of the
  15 corruption types at a time. Hence, in total
  we construct 15 different pairs of autoencoders. Corrupted test images
  are classified using the transfer model learned from examples of the corresponding
  corruption type. 
  \item \textsc{Transfer-Joint}: for training, we  pool all the 15
    corruption types into one set for training a single  pair of autoencoders.
    Test images with all corruption types are classified using the same resulting transfer model.
\end{itemize}
For both \textsc{Transfer-Separate} and \textsc{Transfer-Joint}
we use for training $N_o=60k-15N_c$ examples $x$ from the original MNIST
dataset and $N_c \in \{0, 128, 1024, 30k\}$ examples $\tilde{x}$ for each of
the 15 corruption classes  from the 
C-MNIST training set. We ensure that none of the $\tilde{x}$  used for training
is the corrupted  version of a clean $x$ also used for training.
The case $N_c=0$, means that we only train
one autoencoder consisting of $E$, and $D_1$ which then still can be used as a transfer
model for new, corrupted examples. There is no difference between \textsc{Transfer-Joint}
and \textsc{Transfer-Separate} in this case.
% Note that we 
% discard the MNIST training examples corresponding 
% to the C-MNIST training examples, while training 
% $E$, $D_1$, and $D_2$.

\begin{figure}[h]
\centering
\centerline{ \includegraphics[width=0.5\columnwidth]{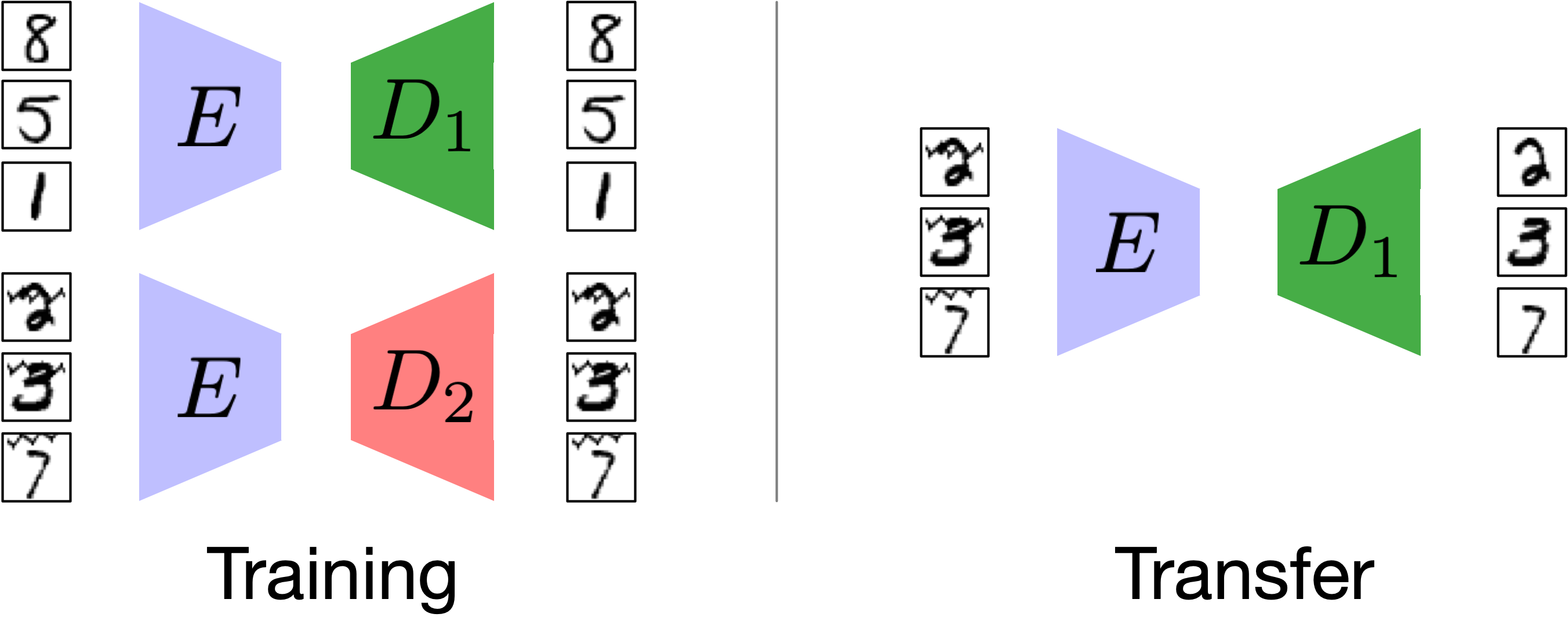}}
\caption{Transfer by paired auto-encoders}
\label{fig:deepfake}
\end{figure}

As baselines we consider \textsc{Plain} classifiers wich are obtained by standard
training on clean images, and then used without modifications to classify noisy images.
As a second baseline we used data augmentation 
by adding Gaussian noise to the original MNIST pictures. The Gaussian noise centered
in 0 with standard deviation of 0.1, was applied independently for each pixel.
NN and SVM classifiers were
then learned from a mix of clean and noisy examples.

Table~\ref{tab:cmnist:svm:nn} shows a summary of results\footnote{More detailed results
are contained in the supplementary material}.
For each classifer we report the accuracy on corrupted test images,
averaged over the 15 corruption types,  and the accuracy for the original MNIST test set
We first observe that adding a robust training approach nowhere leads to a significant
decrease in the clean data accuracy. Both the \textsc{Plain} NN and SVM models are
significantly less accuracte on the noisy than the clean data, where the gap is even
more pronounced for SVM than NN. The \textsc{Gauss} baseline for robust training
in this experiment only led to a minor improvement. This is in contrast to findings
by \cite{rusak2020simple}, who reported  that Gaussian noise augmentation can be
quite effective. However, they also show that a very careful calibration of the Gaussian
noise generation is required. 
The results for both \textsc{Transfer} approaches show a significant improvement of classification
accuracy for the corrupted examples as the number of noisy training examples increases.
This improvement is consistent for the NN and SVM models. Most remarkably, perhaps,
the results for \textsc{Transfer-Separate} and \textsc{Transfer-Joint} are very similar,
showing that a single transfer model $T$ can handle a wide range of different data
transformations. 

\cite{mu2019mnist} reported an accuracy 
of $91.91\%$ for what should correspond exactly to our \textsc{Plain} NN model.
We were unable to reproduce their results, and only obtained the  $86.30\%\pm0.52$
shown in the table.
\cite{rusak2020simple} report and accuracy of $92.2\%$ for their trained Gaussian noise
augmentation approach. However, these results are not directly comparable, since
\cite{rusak2020simple} use a different neural network architecture.

\section{Conclusions}
%TODO \alessandro{This is just a sketch of what we discussed on 27th of January. The instances (Examples~\ref{ex:succrates} and~\ref{ex:onmanifoldaverage}) of our general framework are more effective in terms of robustness score (many other works used a fixed epsilon and upfront the evaluate the success rate).}
We have introduced a general framework for defining robustness of classifiers.
Our framework allows flexible definitions for different types of robustness that in a uniform manner
permit to calibrate robustness objectives with respect to different assumptions on the nature of
adversarial attacks or distribution changes that one wants to protect against. Our framework
establishes a natural link between the recent work in robustness, and more established
work in the area of transfer learning. Considering two specific robustness concepts,   we have
developed new training methods for image classifiers for these robustness objectives.
Our results show particularly promising results in a learning setting where a classifier can
be incrementally trained to become more robust by being exposed to novel, unlabeled examples.

While being greatly motivated by the conceptual framework, our algorithmic solutions are not directly
linked to the formal foundations we laid in this paper. A possible line of future work can be the
development of more generic training methods that can then be simply instantiated for a given
robustness objectives. On the theoretical side, an interesting question would be to what extent
robustness for a given measure $R_{Q,H,G}$ implies lower bounds on the robustness according to
another  measure $R_{Q',H',G'}$, and, thus, whether one type of robustness as a training objective
typically leads to models that are robust in a wider sense.

\clearpage
% \bibliography{references}

\clearpage

\appendix
\title{Supplementary Material}
\maketitle
\section{Details for Section 3.3}

\subsection{Model architectures used in the experiments}
For the experiments, we chose a relatively simple  convolutional neural network, which is considered the state-of-the-art method for image classification (see, e.g.,~\cite{szegedy2017inception}), and it can still provide reasonable high accuracies. The convolutional neural network $f$ architecture used for MNIST~\cite{lecun1998gradient},  SVHN~\cite{netzer2011reading}, and F(ashion)-MNIST~\cite{xiao2017fashion} is depicted in Figure~\ref{fig:arch:f}. $N=28$, $C=1$ for F-MNIST and MNIST, while $N=32$, $C=3$ for SVHN. Each convolutional layer ($5 \times 5$ kernels; stride 1; 128 channels), is followed by batch normalization~\cite{ioffe2015batch} and LeakyReLU~\cite{maas2013rectifier}  activations ($\alpha=0.1$). The two dense layers has $1024$ and $10$ units respectively. The first is followed by batch normalization and LeakyReLU activations; the second is followed by softmax activations). For CIFAR-10~\cite{krizhevsky2009learning} we used for $f$ the ResNet20 architecture described in \cite{he2016deep}. 
\begin{figure}[ht]
\vskip 0.2in
\begin{center}
\centerline{ \includegraphics[width=\columnwidth]{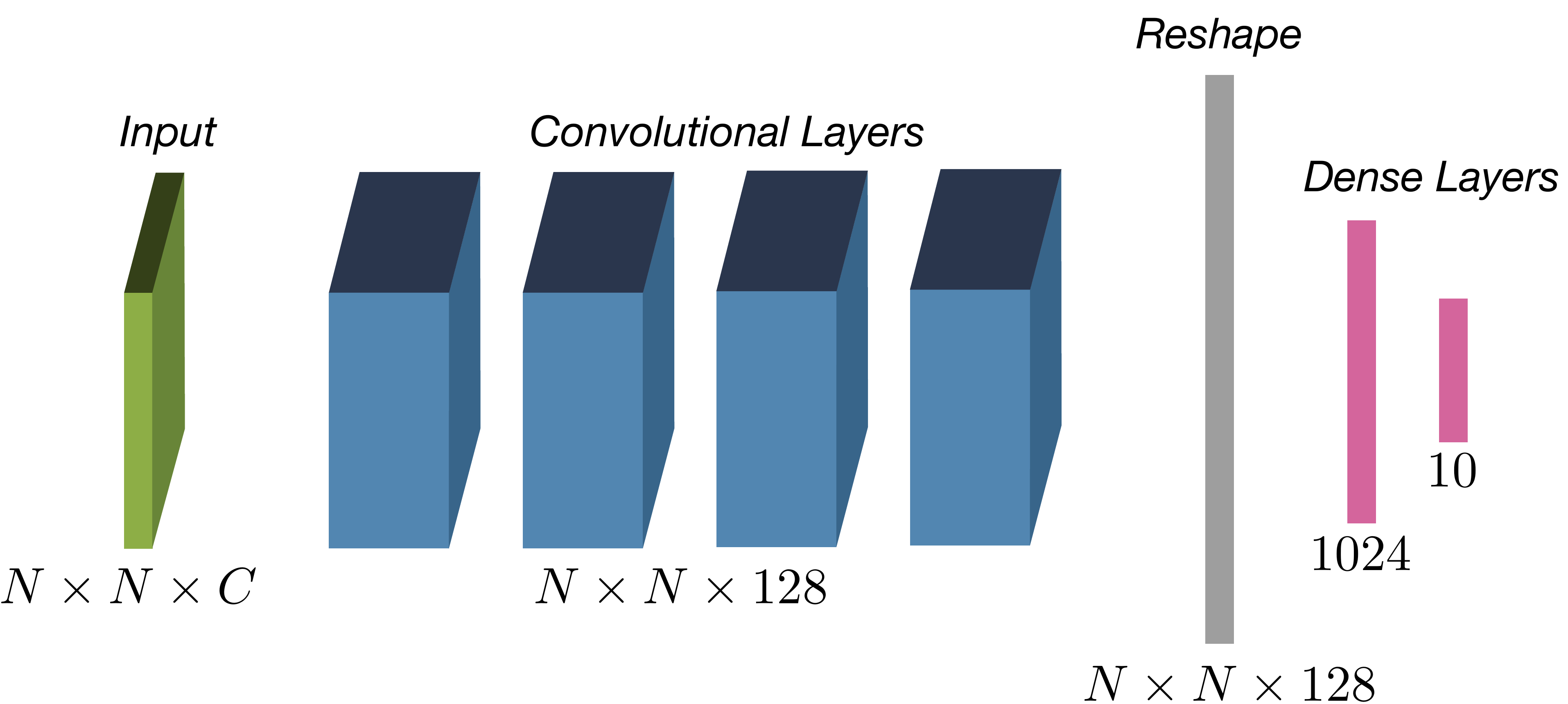}}
\caption{Architecture of $f$ used for F-MNIST, MNIST, and SVHN.}
\label{fig:arch:f}
\end{center}
\vskip -0.2in
\end{figure}

$g$ is another convolutional neural network, whose structure, depicted in Figure~\ref{fig:arch:g}, is the same for all the datasets (with the exception of the filters of the last layer that depend whether the input picture is grayscale or RGB).  $N=28$, $C=1$ for F-MNIST and MNIST, while $N=32$, $C=3$ for SVHN and CIFAR-10. Here again each intermediate convolutional layer ($5 \times 5$ kernels; stride 1; 128 channels), is followed by batch normalization and LeakyReLU activations ($\alpha=0.1$). The output layer is another convolutional layer with sigmoid activations.
\begin{figure}[ht]
\vskip 0.2in
\begin{center}
\centerline{ \includegraphics[width=0.8\columnwidth]{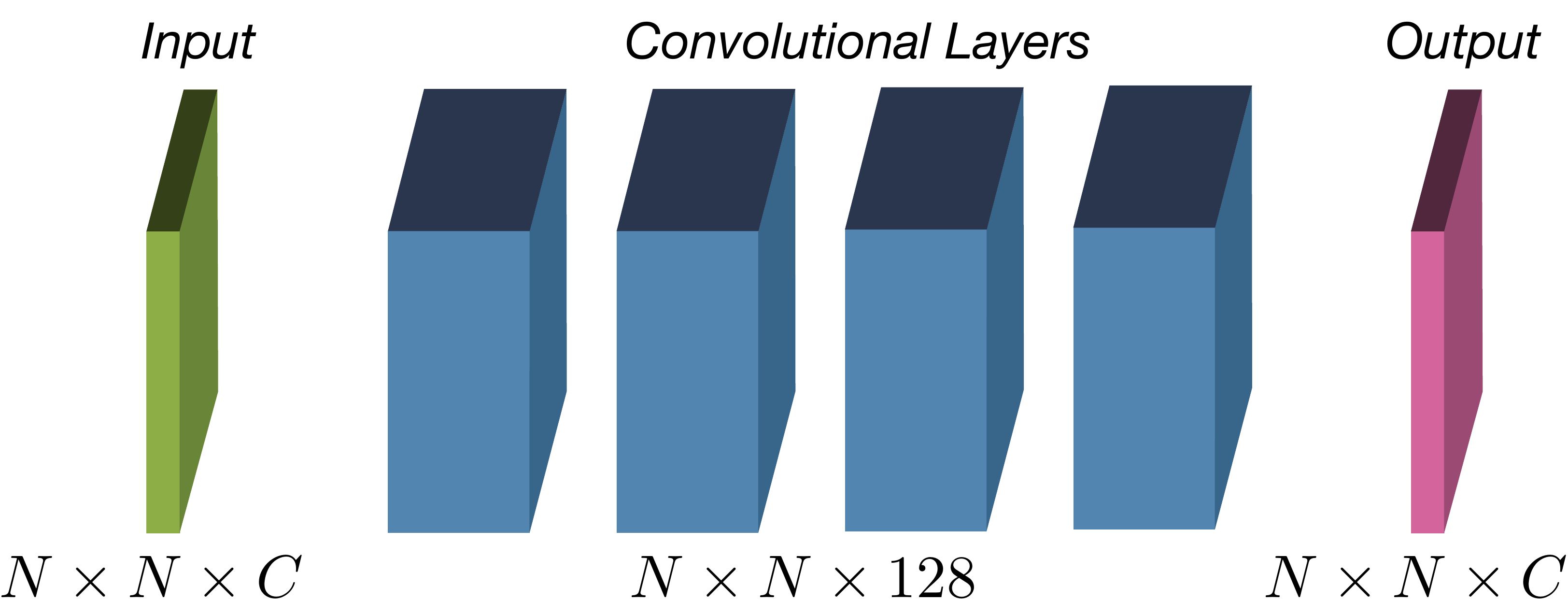}}
\caption{Architecture of $g$ used for F-MNIST, MNIST, SVHN. and CIFAR-10.}
\label{fig:arch:g}
\end{center}
\vskip -0.2in
\end{figure}

For each dataset we reserved $10\%$ of examples as validation set. During the training, we then selected the best model based on the accuracy on the validation set. For F-MNIST, MNIST and SVHN. The weights of $f$ and $g$ are updated with Adam~\cite{kingma2014adam} optimizer on mini-batches of size 32 for 10 epochs. For both $f$ and $g$ the learning rates are set to $0.001$ for the first 5 epochs , and $1\mathrm{e}{-4}$ after. For CIFAR-10 the weights of $f$ and $g$ are updated with Adam optimizer on mini-batches of size 32 for 200 epochs. The learning rates are set to $0.001$ for the first 80 epochs, $1\mathrm{e}{-4}$ until the 120th epoch, $1\mathrm{e}{-5}$ until the 160th epoch, and finally $0.5\mathrm{e}{-6}$ after.
We used the MNIST dataset contained in the tensorflow library (\url{https://www.tensorflow.org}) and the F-MNIST dataset downloadable from \url{https://github.com/zalandoresearch/fashion-mnist}.

\subsection{Additional results}

\begin{table*}[t]
\caption{For each method and dataset, the columns
represent the percentage of closest adversarial examples
to the original examples, after the binary search.}
\label{tab:norms}
\vskip 0.15in
\begin{center}
\begin{small}
\begin{sc}
\begin{tabular}{lccccccccc}
\toprule
  \multirow{2}[2]{*}{Dataset}
& \multicolumn{3}{c}{No Reg}
& \multicolumn{3}{c}{Hein}
& \multicolumn{3}{c}{Ours} \\
\cmidrule(lr){2-4}
\cmidrule(lr){5-7}
\cmidrule(lr){8-10}
 & BG & FGSM & PGD & BG & FGSM & PGD & BG & FGSM & PGD \\
\midrule
%MNIST    & 98.32 & 0.53 & 1.15 & 85.11 &  0.95 & 13.94 & 67.61 &  1.19 & 31.20\\
MNIST    & 97.84 & 0.58 & 1.58 & 71.92 &  0.88 & 27.20 & 13.68 &  0.52 & 85.80\\
CIFAR-10 & 99.78 & 0.00 & 0.22 & 99.96 &  0.00 &  0.04 & 98.77 &  0.00 &  1.23\\
%SVHN     & 94.26 & 2.76 & 2.98 & 93.73 &  3.93 &  2.34 & 93.40 &  3.56 &  3.04\\
SVHN     & 93.53 & 6.38 & 0.09 & 92.78 &  7.21 &  0.01 & 92.78 &  7.21 &  0.01 \\
%F-MNIST  & 94.09 & 3.05 & 2.86 & 93.54 &  3.60 &  2.86 & 92.54 &  3.75 &  3.71\\
F-MNIST  & 93.17 & 6.18 & 0.65 & 92.46 &  7.03 &  0.51 & 82.43 &  6.28 & 11.29\\
\bottomrule
\end{tabular}
\end{sc}
\end{small}
\end{center}
\vskip -0.1in
\end{table*}

Table \ref{tab:norms} reports the percentage of training examples for which the respective methods
led to the final closest adversarial example $x^*$ (for each dataset and for each method).
The results suggest that the BG technique produces in general closer adversarial examples to the original inputs.
It is worth pointing out that this shows that the results reported in Figure~\ref{fig:gopfert} are
mostly based on adversarial examples generated by the BG method, which the {\sc Hein} method is
designed to defend against.

Examples of adversarial examples generated with the BG algorithm and their $L_2$-distance to the original data point
are depicted in Figure \ref{fig:mnist-hein-example}. This illustrates that the adversarial example for our
method is visually significantly more distinct than the adversarial examples for the other two methods.
A larger and systematic set of illustrative examples is given in the supplementary material.

\begin{figure}[h]
\vskip 0.2in
\begin{center}
\centerline{ \includegraphics[width=\columnwidth]{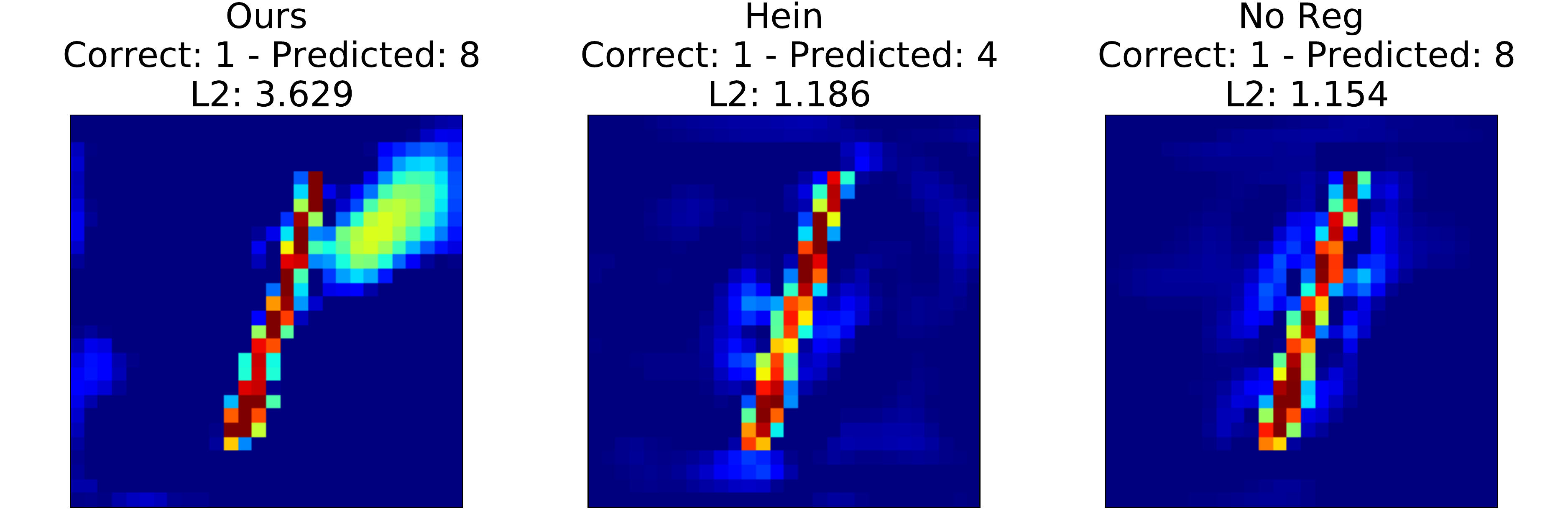}}
\caption{Examples of adversarial examples closed to the margin generated by using ~\cite{hein2017formal} for the three models. The three adversarial examples are generated according to the same input picture. From left to right, the adversarial examples are constructed from the models with the proposed regularization, the regularization proposed by~\cite{hein2017formal}, and no regularization.}
\label{fig:mnist-hein-example}
\end{center}
\vskip -0.2in
\end{figure}

\section{Details for Experiments of Section 4}
Layer notation:
\begin{itemize}
\item Conv2D($c$, $k$, $s$, $a$) / Conv2DTranspose($c$, $k$, $s$, $a$), where $c$ is 
the number of filters, $k$ is the 
kernel size, $s$ is the stride, and $a$ is the activation function;
\item MaxPooling($s$), where $s$ is the stride;
\item Dense($u$,$a$), where $u$ represents the number of units, and $a$ the activation function;
\item Dropout($r$), where $r$ represents the rate;
\end{itemize}

The Neural Network classifer has the following structure:
\begin{eqnarray}
& & \text{Conv2D(32, 3, 1, `relu')} \rightarrow \text{Conv2D(64, 3, 1,`relu')} \rightarrow \nonumber \\
& & \text{MaxPooling(2)} \rightarrow \text{Dropout(0.25)} \rightarrow \text{Flatten()} \rightarrow \nonumber \\
& & \text{Dense(128, `relu')} \rightarrow \text{Dropout(0.5)} \rightarrow \nonumber \\
& & \text{Dense(10, `softmax')} \nonumber 
\end{eqnarray}
We reserved $10\%$ of examples as validation set. During the training, we then selected the best model based on the accuracy on the validation set. The weights of the model are updated with Adam~\cite{kingma2014adam} optimizer on mini-batches of size 32 for 50 epochs. The learning rate is set to $0.001$.

SVM uses the RBF kernel with $\gamma=0.05$ and $C=5$.

The architecture for $E$ is the following:
\begin{eqnarray}
& & \text{Conv2D(32, 3, 2, `relu')} \rightarrow \text{Conv2D(64, 3, 2,`relu')} \rightarrow \nonumber \\
& & \text{Conv2D(128, 3, 2, `relu')} \rightarrow \text{Flatten()} \rightarrow \nonumber \\
& & \text{Dense(128, `linear')}\nonumber
\end{eqnarray}
$D_1$ and $D_2$ share the following architecture:
\begin{eqnarray}
& & \text{Dense(2048, `relu')} \rightarrow \text{Dense(7*7*64, `relu')}  \rightarrow \nonumber\\
& & \text{Reshape((7,7,64))} \rightarrow \nonumber \\
& & \text{Conv2DTranspose(64, 3, 2, `relu')} \rightarrow \nonumber \\
& & \text{Conv2DTranspose(32, 3, 2, `relu')} \rightarrow \nonumber \\
& & \text{Conv2DTranspose(1, 3, 1, `sigmoid')} \rightarrow \nonumber
\end{eqnarray}
$E$, $D_1$, and $D_2$ are jointly trained minimizing the reconstruction loss for 100 epochs. We used Adam with learning rate equals to $1e-3$ and batch size 32.

\begin{table*}
\caption{Classification accuracies for all 
the corruptions of C-MNIST datasets for the neural network model.}
\label{tab:cmnist:nn:extended}
\vskip 0.15in
\begin{center}
\begin{small}
\begin{sc}
\setlength\tabcolsep{1.5pt}
\begin{tabular}{lccccccccc}
\toprule
\multirow{2}{*}{Model} & \multirow{2}{*}{Plain} & \multirow{2}{*}{Gauss} & \multicolumn{4}{c}{Transfer-Separate} & \multicolumn{3}{c}{Transfer-Joint} \\
 \cmidrule(lr){4-7}
 \cmidrule(lr){8-10}
 & & & 0 & 128 & 1024 & 30k & 128 & 1024 & 30k \\
 \cmidrule(lr){1-1}
 \cmidrule(lr){2-2}
 \cmidrule(lr){3-3}
 \cmidrule(lr){4-4}
 \cmidrule(lr){5-5}
 \cmidrule(lr){6-6}
 \cmidrule(lr){7-7}
 \cmidrule(lr){8-8}
 \cmidrule(lr){9-9}
 \cmidrule(lr){10-10}
brightness & 98.8$\pm$0.02 & 75.1$\pm$0.52 & 98.8$\pm$0.02 & 98.6$\pm$0.03 & 98.6$\pm$0.03 & 98.9$\pm$0.04 & 98.7$\pm$0.05 & 98.9$\pm$0.05 & 98.9$\pm$0.01 \\
\midrule
canny edges & 81.2$\pm$0.46 & 79.9$\pm$0.53 & 81.5$\pm$0.46 & 81.1$\pm$0.63 & 81.7$\pm$0.64 & 81.8$\pm$0.54 & 81.5$\pm$0.32 & 81.8$\pm$0.20 & 81.5$\pm$0.20 \\
\midrule
dotted line & 97.2$\pm$0.52 & 98.3$\pm$0.39 & 97.0$\pm$0.52 & 97.3$\pm$0.31 & 97.4$\pm$0.23 & 98.4$\pm$0.08 & 93.4$\pm$0.51 & 95.3$\pm$0.38 & 98.7$\pm$0.22 \\
\midrule
fog & 93.9$\pm$0.60 & 64.6$\pm$0.44 & 93.9$\pm$1.00 & 93.0$\pm$0.72 & 97.3$\pm$0.21 & 98.9$\pm$0.02 & 86.4$\pm$1.48 & 95.2$\pm$0.43 & 95.3$\pm$0.66 \\
\midrule
glass blur & 89.7$\pm$0.48 & 92.8$\pm$0.42 & 89.8$\pm$0.56 & 92.2$\pm$1.48 & 93.9$\pm$0.22 & 93.4$\pm$0.56 & 92.9$\pm$0.08 & 92.6$\pm$0.30 & 92.3$\pm$0.24 \\
\midrule
identity & 99.1$\pm$0.03 & 99.3$\pm$0.71 & 99.1$\pm$0.03 & 99.0$\pm$0.02 & 99.0$\pm$0.04 & 99.1$\pm$0.01 & 99.0$\pm$0.01 & 99.1$\pm$0.05 & 99.1$\pm$0.01 \\
\midrule
impulse noise & 88.6$\pm$0.78 & 93.4$\pm$0.58 & 88.8$\pm$1.12 & 84.7$\pm$2.04 & 90.4$\pm$0.31 & 93.8$\pm$0.68 & 71.4$\pm$0.33 & 92.4$\pm$0.27 & 93.4$\pm$0.22 \\
\midrule
motion blur & 84.7$\pm$1.41 & 95.8$\pm$0.14 & 84.6$\pm$2.11 & 82.3$\pm$1.01 & 91.2$\pm$0.55 & 96.0$\pm$0.28 & 91.9$\pm$0.41 & 93.8$\pm$0.26 & 93.2$\pm$0.32 \\
  \midrule
rotate & 90.7$\pm$0.10 & 92.1$\pm$0.24 & 90.7$\pm$0.10 & 90.1$\pm$0.07 & 91.1$\pm$0.41 & 91.9$\pm$0.09 & 90.5$\pm$0.13 & 91.0$\pm$0.04 & 92.1$\pm$0.07 \\
\midrule
scale & 91.5$\pm$0.23 & 94.3$\pm$0.33 & 91.5$\pm$0.29 & 87.4$\pm$0.87 & 87.6$\pm$0.33 & 94.4$\pm$0.18 & 90.9$\pm$0.47 & 92.7$\pm$0.21 & 94.1$\pm$0.11 \\
\midrule
shear & 97.1$\pm$0.09 & 98.0$\pm$0.50 & 97.1$\pm$0.11 & 95.6$\pm$0.17 & 96.3$\pm$1.44 & 96.9$\pm$0.08 & 96.8$\pm$0.08 & 97.4$\pm$0.07 & 97.5$\pm$0.07 \\
\midrule
shot noise & 97.6$\pm$0.12 & 97.7$\pm$0.33 & 97.6$\pm$0.08 & 91.6$\pm$3.34 & 97.3$\pm$0.31 & 97.3$\pm$0.07 & 97.7$\pm$0.10 & 97.8$\pm$0.08 & 97.8$\pm$0.10 \\
\midrule
spatter & 96.8$\pm$0.06 & 98.4$\pm$0.40 & 96.8$\pm$0.08 & 94.7$\pm$1.38 & 95.1$\pm$0.20 & 97.2$\pm$0.09 & 96.8$\pm$0.16 & 96.9$\pm$0.07 & 96.9$\pm$0.10 \\
\midrule
stripe & 33.8$\pm$4.50 & 95.6$\pm$0.11 & 33.3$\pm$7.50 & 78.1$\pm$1.21 & 90.8$\pm$0.44 & 93.5$\pm$0.46 & 85.2$\pm$2.19 & 93.6$\pm$0.55 & 93.9$\pm$0.42 \\
\midrule
translate & 52.6$\pm$0.12 & 51.8$\pm$0.16 & 52.5$\pm$0.20 & 56.8$\pm$1.19 & 56.5$\pm$0.27 & 56.3$\pm$0.27 & 56.2$\pm$0.17 & 56.8$\pm$0.06 & 53.1$\pm$0.04 \\
\midrule
zigzag & 87.6$\pm$0.32 & 89.5$\pm$0.37 & 87.6$\pm$0.38 & 92.9$\pm$0.42 & 92.7$\pm$0.31 & 96.2$\pm$0.23 & 86.8$\pm$0.48 & 88.4$\pm$0.28 & 95.5$\pm$0.20 \\
\bottomrule
\end{tabular}
\end{sc}
\end{small}
\end{center}
\vskip -0.1in
\end{table*}

\begin{table*}[ht]
\caption{Classification accuracies for all 
the corruptions of C-MNIST datasets for the SVM model.}
\label{tab:cmnist:svm:extended}
\vskip 0.15in
\begin{center}
\begin{small}
\begin{sc}
\setlength\tabcolsep{2.0pt}
\begin{tabular}{lccccccccc}
\toprule
\multirow{2}{*}{Model} & \multirow{2}{*}{Plain} & \multirow{2}{*}{Gauss} & \multicolumn{4}{c}{Transfer-Separate} & \multicolumn{3}{c}{Transfer-Joint} \\
 \cmidrule(lr){4-7}
 \cmidrule(lr){8-10}
 & & & 0 & 128 & 1024 & 30k & 128 & 1024 & 30k \\
 \cmidrule(lr){1-1}
 \cmidrule(lr){2-2}
 \cmidrule(lr){3-3}
 \cmidrule(lr){4-4}
 \cmidrule(lr){5-5}
 \cmidrule(lr){6-6}
 \cmidrule(lr){7-7}
 \cmidrule(lr){8-8}
 \cmidrule(lr){9-9}
 \cmidrule(lr){10-10}
brightness & 69.7 & 97.7$\pm$0.02 & 98.1$\pm$0.09 & 98.0$\pm$0.03 & 97.9$\pm$0.07 & 98.2$\pm$0.03 & 98.1$\pm$0.04 & 98.1$\pm$0.02 & 98.1$\pm$0.06\\
\midrule
canny edges & 69.7 & 34.4$\pm$0.05 & 60.4$\pm$1.12 & 44.3$\pm$1.82 & 72.5$\pm$1.03 & 70.3$\pm$1.14 & 58.5$\pm$1.13 & 62.7$\pm$1.49 & 61.5$\pm$1.18\\
\midrule
dotted line & 69.7 & 97.0$\pm$0.02 & 97.2$\pm$0.22 & 94.1$\pm$0.54 & 95.2$\pm$0.50 & 98.0$\pm$0.08 & 95.5$\pm$0.27 & 96.1$\pm$0.39 & 96.5$\pm$0.23\\
\midrule
fog & 69.7 & 58.8$\pm$0.12 & 94.2$\pm$0.94 & 92.4$\pm$0.65 & 96.8$\pm$0.18 & 98.2$\pm$0.04 & 84.6$\pm$1.56 & 91.6$\pm$1.00 & 91.1$\pm$0.86\\
\midrule
glass blur & 69.7 & 93.8$\pm$0.06 & 90.7$\pm$1.03 & 94.0$\pm$2.35 & 95.0$\pm$0.32 & 96.2$\pm$0.15 & 94.7$\pm$0.34 & 95.6$\pm$0.19 & 95.5$\pm$0.19\\
\midrule
identity & 69.7 & 97.9$\pm$0.02 & 98.3$\pm$0.06 & 98.3$\pm$0.06 & 98.2$\pm$0.02 & 98.3$\pm$0.03 & 98.2$\pm$0.07 & 98.3$\pm$0.05 & 98.3$\pm$0.03\\
\midrule
impulse noise & 69.7 & 25.4$\pm$0.04 & 92.6$\pm$0.52 & 63.7$\pm$3.42 & 89.3$\pm$0.75 & 94.0$\pm$0.64 & 73.9$\pm$2.52 & 91.9$\pm$0.34 & 93.4$\pm$0.65\\
\midrule
motion blur & 69.7 & 85.8$\pm$0.11 & 84.7$\pm$0.77 & 83.0$\pm$0.65 & 89.3$\pm$0.40 & 93.4$\pm$0.45 & 87.3$\pm$0.34 & 88.0$\pm$0.57 & 88.1$\pm$0.70\\
\midrule
rotate & 69.7 & 86.9$\pm$0.02 & 88.2$\pm$0.13 & 87.7$\pm$0.03 & 87.3$\pm$0.20 & 88.2$\pm$0.06 & 88.1$\pm$0.12 & 88.5$\pm$0.08 & 88.5$\pm$0.12\\
\midrule
scale & 69.7 & 51.1$\pm$0.39 & 64.1$\pm$0.75 & 63.5$\pm$0.78 & 65.2$\pm$1.26 & 67.6$\pm$1.09 & 63.7$\pm$1.03 & 64.4$\pm$0.92 & 64.0$\pm$0.84\\
\midrule
shear & 69.7 & 93.8$\pm$0.02 & 94.8$\pm$0.13 & 93.3$\pm$0.16 & 92.4$\pm$0.19 & 94.5$\pm$0.15 & 94.7$\pm$0.12 & 95.0$\pm$0.08 & 95.0$\pm$0.07\\
\midrule
shot noise & 69.7 & 97.1$\pm$0.06 & 97.1$\pm$0.18 & 95.4$\pm$1.20 & 96.6$\pm$0.24 & 97.7$\pm$0.06 & 97.2$\pm$0.13 & 97.5$\pm$0.11 & 97.5$\pm$0.07\\
\midrule
spatter & 69.7 & 96.7$\pm$0.04 & 96.1$\pm$0.28 & 94.7$\pm$1.21 & 95.3$\pm$0.22 & 96.9$\pm$0.04 & 96.6$\pm$0.11 & 96.8$\pm$0.08 & 96.7$\pm$0.07\\
\midrule
stripe & 69.7 & 10.3$\pm$0.00 & 22.4$\pm$4.71 & 77.9$\pm$0.83 & 92.1$\pm$0.41 & 94.5$\pm$0.59 & 88.4$\pm$1.64 & 94.6$\pm$0.66 & 94.8$\pm$0.45\\
\midrule
translate & 69.7 & 21.6$\pm$0.05 & 23.5$\pm$0.15 & 22.5$\pm$0.31 & 21.3$\pm$0.30 & 21.0$\pm$0.18 & 23.3$\pm$0.10 & 23.2$\pm$0.13 & 22.9$\pm$0.22\\
\midrule
zigzag & 69.7 & 80.8$\pm$0.12 & 89.1$\pm$0.78 & 93.3$\pm$0.26 & 94.2$\pm$0.46 & 97.2$\pm$0.11 & 87.5$\pm$1.12 & 84.8$\pm$0.52 & 84.2$\pm$0.68\\
\bottomrule
\end{tabular}
\end{sc}
\end{small}
\end{center}
\vskip -0.1in
\end{table*}

\end{document}